\documentclass[sigconf]{acmart}

\renewcommand\footnotetextcopyrightpermission[1]{} 
\AtBeginDocument{%
  \providecommand\BibTeX{{%
    \normalfont B\kern-0.5em{\scshape i\kern-0.25em b}\kern-0.8em\TeX}}}

\acmDOI{}

\makeatletter
\renewcommand\@formatdoi[1]{\ignorespaces}
\makeatother

\usepackage{algorithm}
\usepackage{algpseudocode}

\usepackage{amsthm}
\theoremstyle{plain}
\newtheorem{theorem}{Theorem}[section]
\newtheorem{definition}[theorem]{Definition}
\newtheorem{example}[theorem]{Example}
\newtheorem{remark}[theorem]{Remark}

\def\*#1{\mathbf{#1}}
\def\~#1{\boldsymbol{#1}}

\begin{document}

\title{Can I Trust the Explanations? Investigating Explainable Machine Learning Methods for Monotonic Models}

\author{Dangxing Chen}
\affiliation{%
  \institution{Kunshan Duke University}
  \city{Kunshan}
  \country{China}}
\email{dangxing.chen@dukekunshan.edu.cn}


\begin{abstract}
  In recent years, explainable machine learning methods have been very successful. Despite their success, most explainable machine learning methods are applied to black-box models without any domain knowledge. By incorporating domain knowledge, science-informed machine learning models have demonstrated better generalization and interpretation. But do we obtain consistent scientific explanations if we apply explainable machine learning methods to science-informed machine learning models? This question is addressed in the context of monotonic models that exhibit three different types of monotonicity. To demonstrate monotonicity, we propose three axioms. Accordingly, this study shows that when only individual monotonicity is involved, the baseline Shapley value provides good explanations; however, when strong pairwise monotonicity is involved, the Integrated gradients method provides reasonable explanations on average. 
\end{abstract}




\maketitle

\pagestyle{plain}

\section{Introduction}

In recent decades, machine learning (ML) models have achieved many successes. In comparison with traditional methods, machine learning models are often capable of increasing accuracy at the expense of black-box functionality. The importance of model explanation is particularly important for highly regulated industries such as the financial sector \cite{OCC2021model}. As an example, the Consumer Financial Protection Bureau (CFPB) confirmed that anti-discrimination law requires companies to provide detailed explanations when denying an application for credit when using machine learning methods \footnote{https://www.consumerfinance.gov/about-us/newsroom/cfpb-acts-to-protect-the-public-from-black-box-credit-models-using-complex-algorithms/}. In response to the growing regulatory requirements, researchers are investigating explainable machine learning methods.

Explainable machine learning (XML) methods have been successfully used in the past to achieve great success in machine learning. Popular used methods include SHapley Additive exPlanations (SHAP) \cite{lundberg2017unified}, Local Interpretable Model-Agnostic Explanations (LIME) \cite{ribeiro2016should}, Integrated Gradients (IG) \cite{sundararajan2017axiomatic}, Anchors \cite{ribeiro2018anchors}, and Sensitivity-based methods \cite{horel2020significance}. As a result of these methods, we have gained a better understanding of how ML models function \cite{boardman2022integrated,bussmann2021explainable,horel2018}.

In this paper, we address the attribution problem, which involves allocating the prediction score of a model for a given input to its base features. The attribution to a base feature can be understood as the importance of the feature in the prediction. Credit scoring, for instance, can utilize attribution to understand how each feature contributes to the credit score. The SHAP and IG have been successfully applied to the attribution problem. Further, their use has been demonstrated to comply with a number of theoretical properties, as discussed in \cite{lundberg2017unified, sundararajan2017axiomatic, sundararajan2020many, lundstrom2022rigorous}. While extensive analyses have been conducted, the majority of results have been based on black-box machine learning models without domain knowledge. 

Modeling and ensuring conceptual soundness require domain knowledge: the true model should be consistent with the underlying theories. A number of studies have demonstrated that physics-informed machine learning (PIML) \cite{karniadakis2021physics, greydanus2019hamiltonian} improved black-box machine learning models in terms of interpretation and accuracy by enforcing conservation laws, for example. Finance and other applications often require monotonicity. A person's credit score should be decreased when there is one more past due balance on the account, for example. It is possible to achieve better generalization and interpretation when monotonicity is successfully enforced \cite{liu2020certified,milani2016fast,you2017deep,repetto2022multicriteria}. Such models can be categorized as finance-informed machine learning models (FIML) or more generally science-informed machine learning models (SIML). In addition, monotonicity is often associated with fairness. As an example, with other factors being equal, a person with more past dues should have a lower credit score. The violation of monotonicity may result in unfair consequences that could cause damage to our society. 

In this paper, we ask the following question. \textbf{Can attribution methods deliver consistent scientific explanations if SIML models contain certain scientific knowledge? If so, to what extent?} Specifically, do attribution methods preserve monotonicity for monotonic ML models?

In the past, monotonicity has been considered for XML methods. In \cite{sundararajan2020many}, it is shown that among different SHAP methods, the baseline Shapley value (BShap) method, which is a generalization of Shapley-Shubik \cite{friedman1999three}, preserves demand individual monotonicity. However, as recently highlighted in \cite{chen2022monotonic,chen2023address,gupta2020multidimensional,cotter2019shape}, individual monotonicity is not the only problem to be addressed; pairwise monotonicity is just as important. The concept of pairwise monotonicity refers to the comparison of features within a given pair. For instance, a past due of more than 60 days should be considered more serious than one of less than 60 days. Pairwise monotonicity is a requirement of fairness that is informed by domain knowledge. It is unfortunate that pairwise monotonicity has been neglected in the existing literature. 

This paper extends monotonicity results to a broader range of cases. As a summary, we have made the following contributions:
\begin{enumerate}
    \item We propose three new axioms concerning the average preservation of three types of monotonicity. Accordingly, a good attribution method should satisfy these axioms if the underlying model is monotonic. 
    \item As we show, IG preserves all average monotonicity axioms, but failing to preserve demand individual monotonicity. 
    \item While BShap preserves individual monotonicity and weak pairwise monotonicity, it fails to preserve average strong pairwise monotonicity. 
\end{enumerate}

A number of examples are used to illustrate our results. BShap is appropriate when only individual monotonicity is considered and is better suited if demand individual monotonicity is required. On the other hand, when strong pairwise monotonicity is involved, IG provides reasonable explanations on average.

The remainder of the paper is organized as follows. Section 2 introduces attributions, Integrated gradients, (baseline) Shapley value, axioms, and types of monotonicity. In Section 3, we propose three monotonic-related axioms and analyze IG and BShap for these axioms. In Section 4, we present an empirical example. Section 5 concludes.

\section{Preliminaries}
\label{prerequisites}
For problem setup, assume we have $\mathcal{D} \times \mathcal{Y}$, where $\mathcal{D}$ is the dataset with $n$ samples and $m$ features and $\mathcal{Y}$ is the corresponding numerical values in regression and labels in classification. For $\*a, \*b \in \mathbb{R}^m$, define $[\*a,\*b]$ to be the hyperrectangle. We denote a class of functions $f: [\*a,\*b] \rightarrow \mathbb{R}$ by $\mathcal{F}(\*a,\*b)$, or simply $\mathcal{F}$.  For simplicity, we assume $\*x \in [\*a, \*b]^m$ and $f$ is differentiable almost everywhere.



\subsection{Attribution}
Following \cite{lundstrom2022rigorous}, we call the point of interest $\*x$ to explain as an explicand and $\*x'$ a baseline. The Baseline Attribution Method that interprets features' importance is defined as follows.

\begin{definition}[Baseline Attribution Method (BAM)]
    Given $\*x, \*x' \in [\*a,\*b]$, $f \in \mathcal{F}(\*a,\*b)$, a baseline attribution method is any function of the form $\*A: [\*a,\*b] \times [\*a,\*b] \times \mathcal{F}(\*a,\*b) \rightarrow \mathbb{R}^m$ with $\*A(\*x',\*x',f) = \*0$. We may also write $\*A$ and denote $A_i$ as the $i$th attribution of $\*A$ for simplicity. 
\end{definition}

\subsubsection{Integrated Gradients}

Integrated Gradients \cite{sundararajan2017axiomatic} is a simple yet useful tool to measure features' importance. 

\begin{definition}[Integrated Gradients (IG)]
    Given $\*x, \*x' \in [a,b]$ and $f \in \mathcal{F}(\*a,\*b)$, the integrated gradients attribution of the $i$-th component of $\*x$ is defined as 
    \begin{align}
        \text{IG}_i(\*x,\*x',f) = (x_i-x_i') \int_0^1 \frac{\partial f}{\partial x_i} \left( \*x' + t(\*x-\*x') \right) \ dt.
    \end{align}
\end{definition}

In practice, this is usually calculated through numerical approximation
\begin{align}
    \text{IG}_i(\*x,\*x',f) \approx (x_i-x_i') \frac{1}{n} \sum_{k=1}^n \frac{\partial f(\*x' + \frac{k}{n} (\*x-\*x'))}{\partial x_i},
\end{align}
with 20 - 300 calls of gradient. 
For simplicity, we often use $\text{IG}_i$ for $\text{IG}_i(\*x,\*x',f)$.

\subsubsection{(Baseline) Shapley Value (BShap)}

The Shapley value \cite{lundberg2017unified} takes as input a set function $v:2^M \rightarrow \mathbb{R}$. The Shapley value produces attributions $s_i$ for each player $i \in M$ that add up to $v(M)$. The Shapley value of a player $i$ is given by:
\begin{align}
    s_i = \sum_{S \subseteq M \backslash i} \frac{|S|! (|M|-|S|-1)!}{M!} (v(S \cup i) - v(S)).
\end{align}
Here, we focus on the Baseline Shapley (Shap), which calculates
\begin{align}
    v(S) = f(\*x_S; \*x'_{M \backslash S}).
\end{align}
That is, baseline values replace the feature's absence. We denote Bshap attribution by $\text{BS}_i(\*x,\*x',f)$ and $\text{BS}_i$ sometimes. 

\subsection{Axioms}

Many desirable characteristics of an attribution technique have been identified in the literature, such as sensitivity, implementation variance, and completeness. Interested readers are referred to \cite{lundstrom2022rigorous,sundararajan2020many,sundararajan2017axiomatic} for detailed discussion. This paper focuses on the axiom of monotonicity. \cite{sundararajan2020many} introduced demand individual monotonicity as one of the desired axioms. The attributions for a feature increase with increasing feature values if the model is monotonic with respect to that feature, all else being fixed, including the baseline value. The paper \cite{sundararajan2020many} showed that BShap is a powerful tool that preserves the monotonicity of demand individual demands and other desirable characteristics. 

\subsection{Individual and pairwise monotonicity}

Three types of monotonicity are discussed here. Without loss of generality, we assume that all monotonic features are monotonically increasing throughout the paper. Suppose $\~{\alpha}$ is the set of all individual monotonic features and  $\neg \~{\alpha}$ its complement, then the input $\*x$ can be partitioned into $\*x = (\*x_{\~{\alpha}}, \*x_{\neg \~{\alpha}})$. Individual monotonicity is defined as follows. 
\begin{definition}[Individual Monotonicity] \label{def:indi_mono}
We say $f$ is individually monotonic with respect to $\*x_{\~{\alpha}}$ if
\begin{align} \label{eq:mono_con1}
 & f(\*x) = f(\*x_{\~{\alpha}}, \*x_{\neg \~{\alpha}}) \leq f(\*x^*_{\~{\alpha}}, 
 \*x_{\neg \~{\alpha}}) = f(\*x^*), \nonumber \\
 & \forall \*x,\*x^* \in [\*a,\*b] \text{ s.t. } \*x_{\~{\alpha}} \leq \*x^*_{\~{\alpha}},
\end{align}
where $\*x_{\~{\alpha}} \leq \*x_{\~{\alpha}}'$ means $x_{\alpha_i} \leq x_{\alpha_i}', \forall{i}$.
\end{definition}

\begin{example}\label{eg:IM}
    In the case of credit scoring, given that $f$ represents the probability of default and $x_{\alpha}$ represents the number of past dues, then $f$ should increase monotonically with respect to $x_{\alpha}$. 
\end{example}

In practice, certain features are intrinsically more important than others. Analog to \eqref{eq:mono_con1}, we partition $\*x = (x_{\beta},x_{\gamma},\*x_{\neg})$. Without sacrificing generality, we assume that $x_{\beta}$ has greater significance than $x_{\gamma}$. Lastly, we require that all features exhibiting pairwise monotonicity also exhibit individual monotonicity. Pairwise monotonicity can be categorized into two types: strong and weak. As a more general definition, weak pairwise monotonicity is presented below. 

\begin{definition}[Weak Pairwise Monotonicity] \label{def:weak_mono}
We say $f$ is weakly monotonic with respect to $x_{\beta}$ over $x_{\gamma}$ if 
\begin{align}\label{eq:mono_con2}
    & f(\*x) = f(x_{\beta},x_{\gamma}+c,\*x_{\neg}) \leq f(x_{\beta}+c,x_{\gamma},\*x_{\neg}) = f(\*x^*), \nonumber \\
    & \forall \*x, \*x^* \in [\*a,\*b] \text{ s.t. } x_{\beta}=x_{\gamma}, c>0. 
\end{align}
\end{definition}
Weak pairwise monotonicity compares the significance of $x_{\beta}$ and $x_{\gamma}$ at the same magnitude. The following is an example. 

\begin{example}\label{eg:WPM}
    An applicant who intends to study STEM in graduate school is required to take the GRE general (170 in math and 170 in verbal) as one of the factors for admission. Assume $f$ is the admitted probability, $x_{\beta}$ is the student's math score, and $x_{\gamma}$ is the student's verbal score. Due to the importance of math in STEM, $f$ should be weakly monotonic with respect to $x_{\beta}$ over $x_{\gamma}$. In this case, it is necessary to fulfill the condition $x_{\beta} = x_{\gamma}$. If the student has the same level of math and verbal skills, it is more desirable to see improvement in math. It is a different story when a student has a strong math score but a weak verbal score. There is often a requirement for a minimum verbal score in schools in order to ensure effective communication. A student who has a strong math score but a very weak verbal score may be able to improve his/her chances of admission more significantly if he/she improves his/her verbal score. 
\end{example}

In addition, there is a stronger condition of pairwise monotonicity, known as strong pairwise monotonicity, which is independent of the condition that $x_{\beta} = x_{\gamma}$. Here is the definition. 

\begin{definition}[Strong Pairwise Monotonicity] \label{def:strong_mono}
We say $f$ is strongly monotonic with respect to $x_{\beta}$ over $x_{\gamma}$ if 
\begin{align}\label{eq:mono_con3}
    & f(\*x) = f(x_{\beta},x_{\gamma}+c,\*x_{\neg}) \leq f(x_{\beta}+c,x_{\gamma},\*x_{\neg}) = f(\*x^*), \nonumber \\
    & \forall \*x, \*x^* \in [\*a,\*b] \text{ s.t. } c>0.
\end{align}
\end{definition}

Accordingly, $x_{\beta}$ is always more important than $x_{\gamma}$ regardless of its value. Below is an example. 

\begin{example}\label{eg:SPM}
 In credit scoring, suppose $f$ calculates the probability of default, $x_{\beta}$ counts the number of past dues that have been outstanding for more than two months, and $x_{\gamma}$ counts the number of past dues that are outstanding between one and two months. As $x_{\beta}$ is always more important than $x_{\gamma}$, $f$ should be strongly monotonic with respect to $x_{\beta}$ over $x_{\gamma}$. The pairwise relationship is strong here since whenever there is an additional past due, a longer past due is always more serious.         
\end{example}

\section{Monotonic Preserving}
In related applications, monotonicity is an important component in domain knowledge. This section discusses axioms derived from different types of monotonicity and analyzes IG and BShap behaviors.

\subsection{Demand Individual Monotonicity} 

 Consider two explicands $\*x = (x_1, \dots, x_m)$ and $\*x^* = (x_1, \dots, x_{\alpha}+c, \dots, x_m)$, where $c >0$. We restate the axiom of demand individual monotonicity proposed in \cite{friedman1999three}.

\begin{definition}[\textbf{Demand Individual Monotonicity (DIM)}]\label{def:DIM}
    Suppose $f$ is individually monotonic with respect to $x_{\alpha}$. We say a BAM preserves demand individual monotonicity if 
    \begin{align}
        A_{\alpha}(\*x^*,\*x',f) \geq A_{\alpha}(\*x,\*x',f), \forall \*x, \*x^* \in [\*a,\*b].
    \end{align}
\end{definition}

\begin{example}\label{eg:DIM}
    Consider Example~\ref{eg:IM}, whenever there is an additional past due, the attribution of $x_{\alpha}$ should be increased since $x_{\alpha}$ caused the change. 
\end{example}

As discussed in \cite{sundararajan2020many}, BShap preserves DIM. It is presented here without the contents of the cost-sharing problem. 
\begin{theorem}
    BShap preserves DIM. 
\end{theorem}
\begin{proof}
    Without loss of generality, suppose $f$ is individually monotonic with respect to $x_{1}$. Suppose $\*x = (x_1, \*x_{\neg})$ and $\*x^* = (x_1+c, \*x_{\neg})$ for $c>0$. Then
    \begin{align*}
        & \text{BS}_1(\*x^*,\*x',f) - \text{BS}_1(\*x,\*x',f) \\
        &= \sum_{S \subseteq M \backslash 1} \frac{|S|! (|M|-|S|-1)!}{M!} \Delta f_S,
    \end{align*}
    where
    \begin{align*}
        \Delta f_S = f(x_1+c; \*x_S, \*x_{M \backslash (S \cup 1)}') - f(x_1;\*x_S,\*x_{M \backslash (S \cup 1)}') \geq 0
    \end{align*}
    because of individual monotonicity. Thus, we conclude. 
\end{proof}

IG has unfortunately produced a negative result.

\begin{remark}\label{remark:IG_DIM_fail}
    IG doesn't preserve DIM.
\end{remark}

The following example in \cite{friedman1999three} provides a counterexample with a comparison between IG and BShap. 

\begin{example}\label{eg:DIM_calc}
    Consider the function $f(x_1,x_2) = \frac{x_1 x_2}{x_1+x_2}$ with baseline $\*x' = (0,0)$. For IG, it can be shown that $\text{IG}_1(\*x,\*x',f) = \frac{x_1x_2^2}{(x_1+x_2)^2}$. Note $\frac{\partial}{\partial x_1} \text{IG}_1 = \frac{x_2^2(x_2-x_1)}{(x_1+x_2)^3}$, therefore DIM is not preserved. It is important to note that the path for the new explicand has been changed. As a result, the new path cannot guarantee greater attributions for the explicand. 

    For BShap, we have $\text{BS}_1(\*x,\*x',f) = \frac{x_1 x_2}{2(x_1+x_2)}$, which preserves DIM. The main difference is that in the BShap calculation, the new path covers the old path.

    

\end{example}

\subsection{Average Individual Monotonicity}

DIM concerns individual monotonicity for each change. The average result is considered here as a weaker case. The following axiom is motivated by this fact. 

\begin{definition}[\textbf{Average Individual Monotonicity (AIM)}]
    Suppose $f$ is individually monotonic with respect to $x_{\alpha}$, then we say a BAM preserves average individual monotonicity if 
    \begin{align}
        A_{\alpha}(\*x,\*x',f) \geq 0, \forall \*x \in [\*a,\*b] \text{ s.t. } \*x \geq \*x'.
    \end{align}
\end{definition}

\begin{example}
    Consider Example~\ref{eg:IM}, $x_{\alpha}$ should always provide positive contributions. 
\end{example}

Clearly, DIM is a stronger condition, as stated in the following. 

\begin{theorem}
    If a BAM preserves DIM, then it preserves AIM. 
\end{theorem}
\begin{proof}
    Suppose we want to explain $\*x^*$, let $\*x = \*x'$ in Definition~\ref{def:DIM}, then it follows from the definition. 
\end{proof}

As a result, we have the following corollary. 

\begin{corollary}
    BShap preserves AIM. 
\end{corollary}

In spite of the fact that the IG fails to preserve DIM, it preserves AIM. 

\begin{theorem}
    IG preserves AIM. 
\end{theorem}
\begin{proof}
    Suppose $f$ is individually monotonic with respect to $x_{\alpha}$, then 
    \begin{align*}
            \text{IG}_{\alpha} &= (x_{\alpha} - x_{\alpha}') \int_0^1 \frac{\partial f}{\partial x_{\alpha}} \left( \*x' + t(\*x-\*x') \right) \ dt \geq 0,
    \end{align*}
    since $x_{\alpha} \geq x_{\alpha}'$ and $\frac{\partial f}{\partial x_{\alpha}} \geq 0$.
\end{proof}

As shown in Example~\ref{eg:DIM_calc}, if we consider a new path, IG will still contribute positively to the feature, but the contribution of the features may decrease. To this end, we ask: does the failure of DIM pose a significant concern? DIM is derived from a somewhat implicit baseline of $\*x$. It is important to recall that the increase is relative to $\*x$. However, the attribution we seek to calculate is essentially an average calculation. Consider the following example. 

\begin{example}{\label{eg:DIM_import}}
     Consider Example~\ref{eg:SPM} with a concrete example, where $f(x_1,x_2) = \min(0.2x_1+0.1x_2,0.3)$ and $x_1$ is strongly monotonic with respect to $x_2$. At $\*x = (1,1)$ with baseline $\*x' = (0,0)$, the corresponding attribution should be $\*A(\*x,\*x',f) = (0.2,0.1)$, due to linearity. Now we consider $\*x^* = (3,1)$, $f(\*x^*) = 0.3$ and DIM requires that $A_1(\*x^*,\*x',f)\geq 0.2$. In this case, we wish to satisfy this requirement. However, if we consider $\*x^{!} = (0,4)$ with $f(x^{!}) = 0.3$. From the perspective of $x^!$, $x_1$ does not incur an additional cost. Accordingly, a unit of $x_1$ and $x_2$ should be attributed equally. This results in attributions at $(2,2)$ of $(0.15,0.15)$, which violates the principle of DIM. 
\end{example}

\begin{remark}
    In light of Example~\ref{eg:DIM_import}, we argue that DIM may not be necessary in some cases. Although DIM makes sense for some features, it can be achieved through the use of a historical baseline, as in \cite{alam2022applications}. Taking into consideration the time at which there is an increase in monotonic features of $\*x_t$ and we have $\*x_{t+1}$, then AIM provides positive attributions for $A(\*x_{t+1},\*x_t,f)$, which demand individual monotonicity. 
\end{remark}

\subsection{Average Weak Pairwise Monotonicity}

Weak pairwise monotonicity compares the importance of features of the same magnitude. Therefore, if two features have the same magnitude, we would expect the more important feature to have more attributions. In light of this, we propose the following axiom. 

\begin{definition}[\textbf{Average Weak Pairwise Monotonicity (AWPM)}]
    Suppose $f$ is weakly monotonic with respect to $x_{\beta}$ over $x_{\gamma}$, $x_{\beta}>x_{\beta}'$ and $x_{\gamma} > x_{\gamma}'$. Suppose for an explicand $\*x \in [\*a,\*b]$, we have $x_{\beta} = x_{\gamma}$. Then we say a BAM preserves average weak pairwise monotonicity if 
    \begin{align}
         \frac{1}{x_{\beta} - x_{\beta}'} A_{\beta}(\*x,\*x',f) \geq  \frac{1}{x_{\gamma} - x_{\gamma}'} A_{\gamma}(\*x,\*x',f).
    \end{align}
\end{definition}

\begin{example}
    Consider Example~\ref{eg:WPM}, math scores should contribute more attributions when a student has the same verbal and math scores. 
\end{example}

IG and BShap both preserve AWPM with appropriate baselines. For simplicity, we only show results for $\*x'=\*0$ here. 

\begin{theorem}
    For $\*x' = \*0$, IG preserves AWPM. 
\end{theorem}
\begin{proof}
Suppose $f$ is weakly monotonic with respect to $x_{\beta}$ over $x_{\gamma}$. For $\*x'=\*0$,  we have $\frac{1}{x_{\beta}-x_{\beta}'} = \frac{1}{x_{\gamma}-x_{\gamma}'}$. Since $x_{\beta}=x_{\gamma}$ in $\*x$ and $\*x'$, $x_{\beta}=x_{\gamma}$ for $\*x'+t(\*x-\*x')$, $\forall t \in [0,1]$. $f$ is weakly monotonic respect to $x_{\beta}$ over $x_{\gamma}$, therefore $\frac{\partial f}{\partial x_{\beta}}(\*x) \geq \frac{\partial f}{\partial x_{\gamma}}(\*x)$ if $x_{\beta}=x_{\gamma}$ in $\*x$. Hence, 
\begin{align*}
    \text{IG}_{\beta} &= (x_{\beta} - x_{\beta}') \int_0^1 \frac{\partial f}{\partial x_{\beta}} \left( \*x' + t(\*x-\*x') \right) \ dt \\
    & \geq (x_{\beta} - x_{\beta}') \int_0^1 \frac{\partial f}{\partial x_{\gamma}} \left( \*x' + t(\*x-\*x') \right) \ dt \\
    &=  (x_{\gamma} - x_{\gamma}') \int_0^1 \frac{\partial f}{\partial x_{\gamma}} \left( \*x' + t(\*x-\*x') \right) \ dt  \\
    &= \text{IG}_{\gamma}.
\end{align*}
\end{proof}

\begin{theorem}
    For $\*x' = \*0$,  BShap preserves AWPM. 
\end{theorem}
\begin{proof}
    Without loss of generality, suppose $f$ is weakly monotonic with respect to $x_1$ over $x_2$. For $\*x'=\*0$,  we have $\frac{1}{x_1-x_1'} = \frac{1}{x_2-x_2'}$. Recall that 
    \begin{align*}
        \text{BS}_1 = \sum_{S \subseteq M \backslash 1} \frac{|S|! (|M|-|S|-1)!}{M!} (v(S \cup 1) - v(S)).
    \end{align*}
    For $\text{BS}_2$, we use the symmetry, for each $S$ here, we consider $S'$ such that $x_1$ and $x_2$ are swapped within $S$, and everything else is left unchanged. That is, if $2 \notin S$, $S'=S$; if $2 \in S$, then $S' = (S \backslash 2) \cup 1$. For this arrangement, we only need to show that $v(S \cup 1) - v(S) \geq v(S' \cup 2) - v(S')$ in the summation. If $2 \notin S$, we have
    \begin{align*}
         v(S \cup 1) - v(S) &= f(x_1,0; \*x_{S}, \*x_{M \backslash (S \cup \{1,2\})}’) \\ 
         &- f(0,0; \*x_{S}, \*x_{M \backslash (S \cup \{1,2\}}’), \\
        v(S' \cup 2) - v(S') &= 
        f(0,x_2; \*x_{S'}, \*x_{M \backslash (S' \cup \{1,2\})}’) \\
        &- f(0,0; \*x_{S'}, \*x_{M \backslash (S' \cup \{1,2\})}’).
    \end{align*}
    Since $S=S'$ and $x_1=x_2$, by the definition of weak pairwise monotonicity, we have 
    \begin{align*}
        f(x_1,0; \*x_{S}, \*x_{M \backslash (S \cup \{1,2\})}’) \geq f(0,x_2; \*x_{S'}, \*x_{M \backslash (S' \cup \{1,2\})}’),
    \end{align*}
    and therefore $v(S \cup 1) - v(S) \geq v(S' \cup 2) - v(S')$.
    
    If $2 \in S$,  and we have
    \begin{align*}
        v(S \cup 1) - v(S) &= f(x_1,x_2; \*x_{S \backslash 2}, \*x_{M \backslash (S \cup 1)}’) \\
        &- f(0,x_2; \*x_{S \backslash 2}, \*x_{M \backslash (S \cup 1)}’), \\
        v(S' \cup 2) - v(S') &= f(x_1,x_2; \*x_{S' \backslash 1}, \*x_{M \backslash (S' \cup 2)}’) \\
        &- f(x_1,0; \*x_{S' \backslash 1}, \*x_{M \backslash (S' \cup 2)}’).
    \end{align*}
     Since $S \backslash 2 = S' \backslash 1$, $S \cup 1 = S' \cup 2$, and $x_1=x_2$, by the definition of weak pairwise monotonicity, we have 
     \begin{align*}
         f(0,x_2; \*x_{S \backslash 2}, \*x_{M \backslash (S \cup 1)}’) \leq f(x_1,0; \*x_{S' \backslash 1}, \*x_{M \backslash (S' \cup 2)}’),
     \end{align*} 
     and therefore $v(S \cup 1) - v(S) \geq v(S' \cup 2) - v(S')$. 

     Since $v(S \cup 1) - v(S) \geq v(S' \cup 2) - v(S')$ for all $S$ with corresponding $S'$, we conclude that $\text{BS}_1 \geq \text{BS}_2$.

\end{proof}

This result also suggests baseline points. We need $x_{\beta}' = x_{\gamma}'$ in $\*x'$ in order to preserve AWPM. 
\begin{example}
    Consider $f(x_1,x_2) = 4.5x_1-x_1^2 + 4x_2 - x_2^2$ for $x_1, x_2$ in $[0,2]$. $f$ is weakly monotonic with respect to $x_1$ over $x_2$. Consider the baseline $\*x'=(0,0)$ and the explicand $\*x = (2,2)$. 
    
    For IG, we have
    \begin{align*}
        \textbf{IG}((2,2),(0,0),f) = 
        \left[ \begin{matrix} 5 \\ 4 \end{matrix} \right].
    \end{align*}
    It can be seen that $ \frac{1}{2} \text{IG}_1 > \frac{1}{2} \text{IG}_2$, IG preserves AWPM. Now suppose we consider $\*x' = (1,0)$, we have
    \begin{align*}
        \textbf{IG}((2,2),(1,0),f) = 
        \left[ \begin{matrix} 1.5 \\ 4 \end{matrix} \right]. 
    \end{align*}
    Since $\text{IG}_1 < \frac{1}{2} \text{IG}_2$, IG fails to preserve AWPM. The result of BShap is the same as that of $f$ being additively separable. 

    Consider Example~\ref{eg:WPM}. For $x_{\beta}'=x_{\gamma}'=0$, the weak pairwise monotonicity in $f$ guarantees that math is more important than verbal when they are equal when we determine the general importance of features. If, on the other hand, we take the average historical statistics of admitted students as our baseline, then it would be possible to see that verbal skills are more important than mathematics skills as admitted students are already competent in mathematics. It is important to keep in mind in this case that we are comparing the cases of previously admitted students, and the importance of features does not necessarily apply in general. 
    \end{example}

\subsection{Average Strong Pairwise Monotonicity}

When there is strong pairwise monotonicity, one feature is always more important than another. Consequently, when considering attributions, the average attribution of the more important ones should always be greater. In light of this, we should consider the following axiom. 

\begin{definition}[\textbf{Average Strong Pairwise Monotonicity (ASPM)}]
    Suppose $f$ is strongly monotonic with respect to $x_{\beta}$ over $x_{\gamma}$, $x_{\beta}>x_{\beta}'$, $x_{\gamma} > x_{\gamma}'$, and $\*x \in [\*a,\*b]$. Then we say a BAM preserves average strong pairwise monotonicity if 
    \begin{align}
        \frac{1}{x_{\beta}-x_{\beta}'} A_{\beta}(\*x,\*x',f) \geq \frac{1}{x_{\gamma}-x_{\gamma}'} A_{\gamma}(\*x,\*x',f).
    \end{align}
\end{definition}

\begin{example}
    Consider Example~\ref{eg:SPM}. As a general rule, we should always expect $x_{\beta}$ to contribute more than $x_{\gamma}$ on average, since $x_{\beta}$ is of greater importance. As an alternative, the interpretation may mislead people to believe that payments between one and two months past due can be more serious. 
\end{example}

Clearly, AWPM preserving is a special case of ASPM for $x_{\beta} = x_{\gamma}$. Thus, we have the following theorem.  

\begin{theorem}
    If a BAM preserves ASPM, it also preserves AWPM for the same features. 
\end{theorem}

As a result of IG, ASPM is preserved and we provide the results.

\begin{theorem}\label{thm:IG_ASPM}
    IG preserves ASPM. 
\end{theorem}

\begin{proof}
    Suppose $f$ is strongly monotonic with respect to $x_{\beta}$ over $x_{\gamma}$, then $\frac{\partial f}{\partial x_{\gamma}}(\*x) \leq \frac{\partial f}{\partial x_{\beta}}(\*x)$, $\forall \*x \in [\*a,\*b]$. Hence,
    \begin{align*}
        \frac{1}{x_{\beta}-x_{\beta}'} \text{IG}_{\beta}  &= \int_0^1 \frac{\partial f}{\partial x_{\beta}} \left( \*x' + t(\*x-\*x') \right) \ dt  \\
        & \geq \int_0^1 \frac{\partial f}{\partial x_{\gamma}} \left( \*x' + t(\*x-\*x') \right) \ dt \\
        &= \frac{1}{x_{\gamma} - x_{\gamma}'} \text{IG}_{\gamma}.
    \end{align*}
\end{proof}

BShap doesn't preserve ASPM, as we documented here.  
\begin{remark}\label{remark:BShap_ASPM_fail}
    BShap doesn't preserve ASPM.
\end{remark}

A counterexample is provided here, as well as comparisons with the IG. 

\begin{example}
     Consider the function with two features $f(x_1,x_2) = \log(1+10x_1+9x_2)$, where $f$ is strongly monotonic with respect to $x_1$ over $x_2$. This example mimics the diminishing marginal effect, that is, the impact of $f$ decreases as $x_1+x_2$ increases. Consider the explicand $\*x = (4,1)$ and baseline $\*x' = (0,0)$. 
     
     By calculations, we have $\textbf{BS} \approx \left[ \begin{matrix} 2.7 \\ 1.3 \end{matrix} \right]$. It is evident that ASPM is not preserved. BShap requires a stronger condition to preserve the ASPM, for example, if the averaged change of $x_{\beta}$ always has a greater impact on $f$. Such a condition, however, is not realistic in applications such as credit scoring. In addition, we apply the IG, and we obtain $\textbf{IG} \approx \left[ \begin{matrix} 3.2 \\ 0.7 \end{matrix} \right]$. ASPM is successfully maintained. 

    Consider Example~\ref{eg:SPM}. Using BShap, at first glance, longer past dues appear to be contributing more to the result. On second thought, it seems that averaged longer past dues make a smaller contribution. Providing the customer with this interpretation would cause the customer to incorrectly believe that longer past dues have less impact, leading to a catastrophic outcome. 
     
\end{example}

\begin{remark}
    As discussed in Lemma 4.8 in \cite{sundararajan2020many}, BShap fits in the SHAP framework. Therefore, SHAP doesn't preserve ASPM as well. 
\end{remark}

\section{Empirical Examples}

\subsection{Data description}

We use the Kaggle credit score dataset \footnote{https://www.kaggle.com/c/GiveMeSomeCredit/overview}. For simplicity, data with missing variables are removed. Past dues greater than four times are truncated. The dataset is randomly partitioned into $75\%$ training and $25\%$ test sets. The dataset contains 9 features as explanatory variables, whereas the feature age is excluded from the original dataset to avoid potential discrimination. 
\begin{itemize}
\item $x_1$: Number of times borrower has been 90 days or more past due (integer). 
\item $x_2$: Number of times borrower has been 60-89 days past due but no worse in the last 2 years (integer).
\item $x_3$: Number of times borrower has been 30-59 days past due but no worse in the last 2 years (integer).
\item $x_4$: Total balance on credit cards and personal lines of credit except for real estate and no installment debt such as car loans divided by the sum of credit limits (percentage).
\item $x_5$: Monthly debt payments, alimony, and living costs divided by monthly gross income (percentage). 
\item $x_6$: Monthly income (real).
\item $x_7$: Number of open loans (installments such as car loan or mortgage) and lines of credit (e.g., credit cards) (integer). 
\item $x_8$: Number of mortgage and real estate loans including home equity lines of credit (integer). 
\item $x_{9}$: Number of dependents in the family, excluding themselves (spouse, children, etc.) (integer). 
\item $y$: Client's behavior; 1 = Person experienced 90 days past due delinquency or worse.
\end{itemize}
For simplicity, we consider the zero baseline as $\*x' = \*0$.

\subsection{Model Performance}

We apply black-box highly accurate fully-connected neural network (FCNN)  and transparent monotonic groves of neural additive model (MGNAM) \cite{chen2023address}. The FCNN is considered here due to its high degree of accuracy. The MGNAM model was chosen because (i) it is transparent, and therefore easy to analyze; (ii) monotonicity has been incorporated into the model. Accuracy is measured by the area under the curve (AUC). 

For FCNN, we use one hidden layer with ten neurons. For  MGNAM, we consider $g(\mathbb{E}[y|\*x]) = f(\*x)$ with the architecture
\begin{align}
    f(\*x) = f_{1,2,3}(x_1,x_2,x_3) + f_4(x_4) + \dots + f_9(x_9).
\end{align}
In other words, $x_1-x_3$ are grouped together, and the remaining features are handled using 1-dimensional functions. As a convenience, we only impose monotonicity on $x_1-x_3$. All functions are approximated by neural networks with one hidden layer and two neurons. For $x_1-x_3$, we enforce strong pairwise monotonicity. As a result of the domain knowledge, features that are longer past due should have greater importance. In this paper, we focus on simple architectures since there is no apparent improvement in accuracy for more complicated models, and accuracy is not the primary concern. 

Both models have an AUC of approximately $80\%$. This indicates that transparent models are capable of modeling this dataset. We demonstrate the application of IG and BShap to black-box FCNNs and then provide a more detailed analysis of transparent MGNAMs regarding pairwise monotonicity.

\subsection{FCNN Explanations}

We demonstrate the application of IG and BShap to black-box FCNN. We demonstrate how to interpret, debug, and avoid potential pitfalls with these tools.  
An example of the application of IG and BShap is presented for
\begin{align}
    \*x = \left[ \begin{matrix} 3 & 1 & 1 & 1 & 0.11 & 6250 & 3 & 0 & 2 \end{matrix} \right].
\end{align}
Attributions for $\text{IG}(\*x,\*x',\text{FCNN})$ and $\text{BShap}(\*x,\*x',\text{FCNN})$ are given in Figure~\ref{fig:IG_eg2} and Figure~\ref{fig:Bshap_eg2}. In general, they produce similar results. The following is a rough explanation of the result. Three times past due accounts with a lag time of more than 90 days significantly increase the risk of default. Further, each past due between 60-89 days and 30-59 days also increases the probability of default, but at a smaller magnitude. Having a positive monthly income decreases the probability and having two dependents increases it. 

As a first step, we consider average monotonicity preservation. AIM is preserved by both IG and BShap. In the case of pairwise monotonicity revealed by IG, however, average attributions of $x_1$ are lower than those of $x_2$. Based on Theorem~\ref{thm:IG_ASPM}, this indicates that strong pairwise monotonicity is violated here. Therefore, this FCNN should not be used for prediction in this scenario. Although BShap also observes a similar pattern, it cannot guarantee a violation of strong pairwise monotonicity, as indicated in Remark~\ref{remark:BShap_ASPM_fail}. In this regard, IG has an advantage over pairwise monotonicity. If ASPM is violated in IG, then there must be a problem with the model.

\begin{figure}
    \centering
    \includegraphics[scale=0.4]{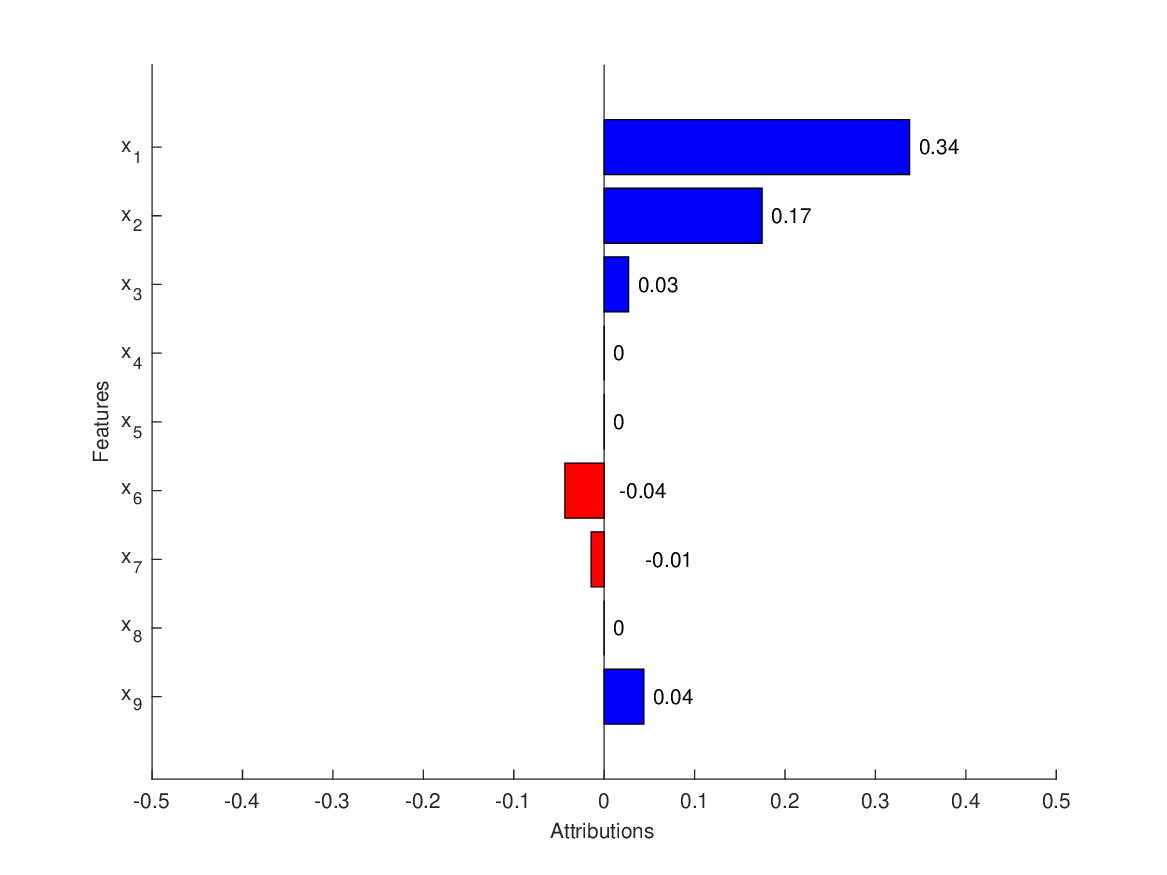}
    \caption{Instance explanation by IG for FCNN.  }
    \label{fig:IG_eg2}
\end{figure}

\begin{figure}
    \centering
    \includegraphics[scale=0.4]{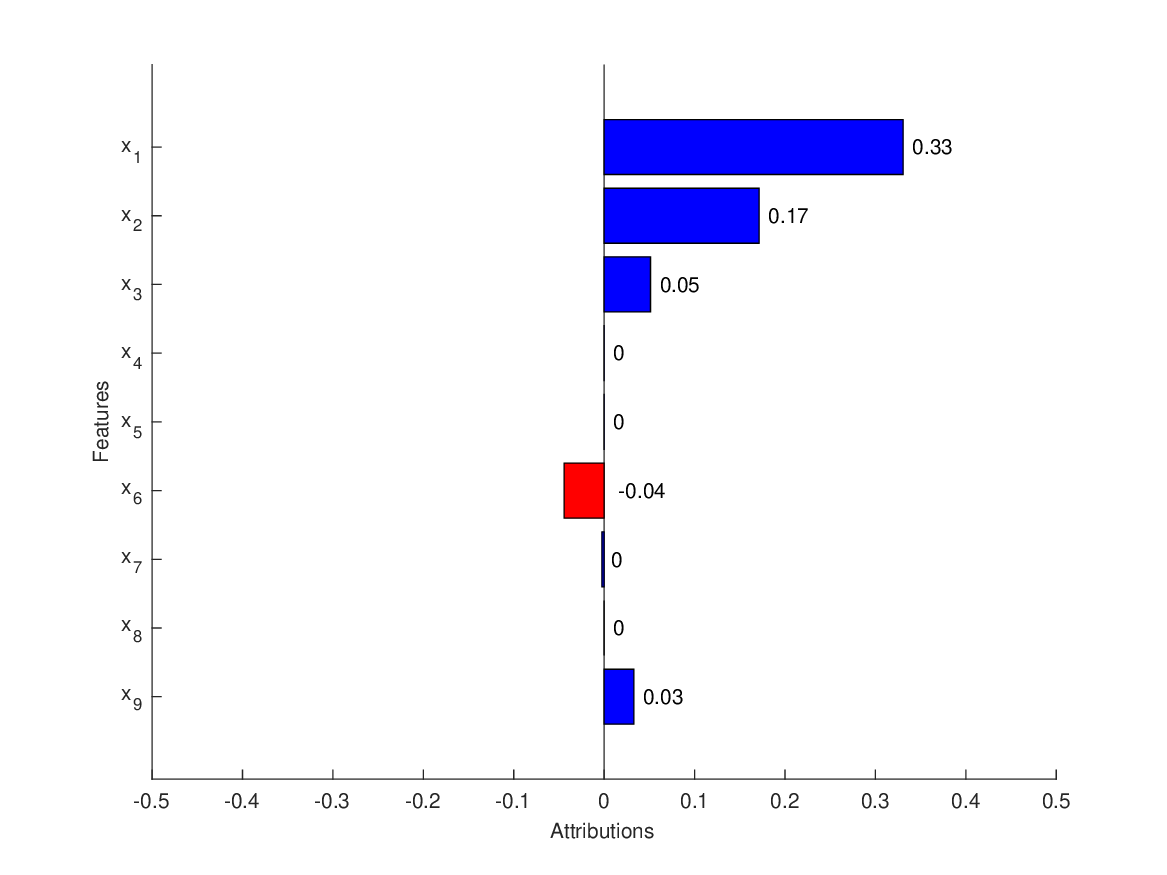}
    \caption{Instance explanation by BShap for FCNN.  }
    \label{fig:Bshap_eg2}
\end{figure}

Next, we consider DIM preserving. We perturb $\*x$ by increasing $x_3$ by one unit,
\begin{align}
    \*x^* = \left[ \begin{matrix} 3 & 1 & 2 & 1 & 0.11 & 6250 & 3 & 0 & 2 \end{matrix} \right].
\end{align}
In principle, we would expect to observe an increase in the probability of default. There is, however, a decrease of 0.015 in the probability of default. As a result, the model does not adequately preserve individual monotonicity. We then recompute IG and BShap in Figure~\ref{fig:IG_eg3} and Figure~\ref{fig:Bshap_eg3}. IG and BShap continue to provide positive attributions, suggesting that the model performs well on average. Note that $\text{IG}_3(\*x^*,\*x',\text{FCNN})$ decreases. In spite of the fact that IG indicates that there may be a problem with the model, we cannot conclude whether a problem exists, based on Remark~\ref{remark:IG_DIM_fail}. Alternatively, $\text{BS}_3(\*x^*,\*x',\text{FCNN})$ increases, despite the decrease in the probability of default. BShap, unfortunately, does not identify the model violation. As a result of these results, although model explanations could help us debug, they cannot be relied upon entirely, and it would be better to begin with the monotonic model instead. 

\begin{figure}
    \centering
    \includegraphics[scale=0.4]{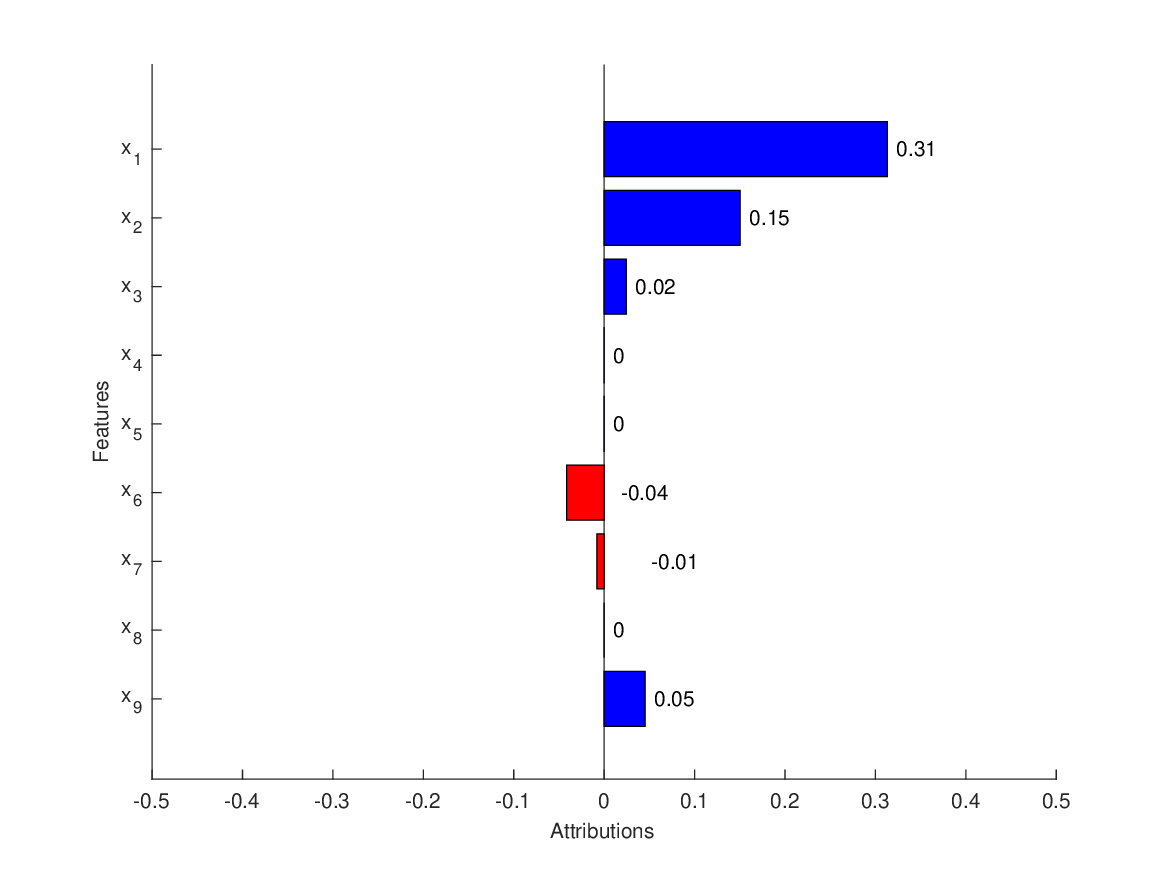}
    \caption{Instance explanation for the perturbed explicand (adding one $x_3$) by IG for FCNN. }
    \label{fig:IG_eg3}
\end{figure}

\begin{figure}
    \centering
    \includegraphics[scale=0.4]{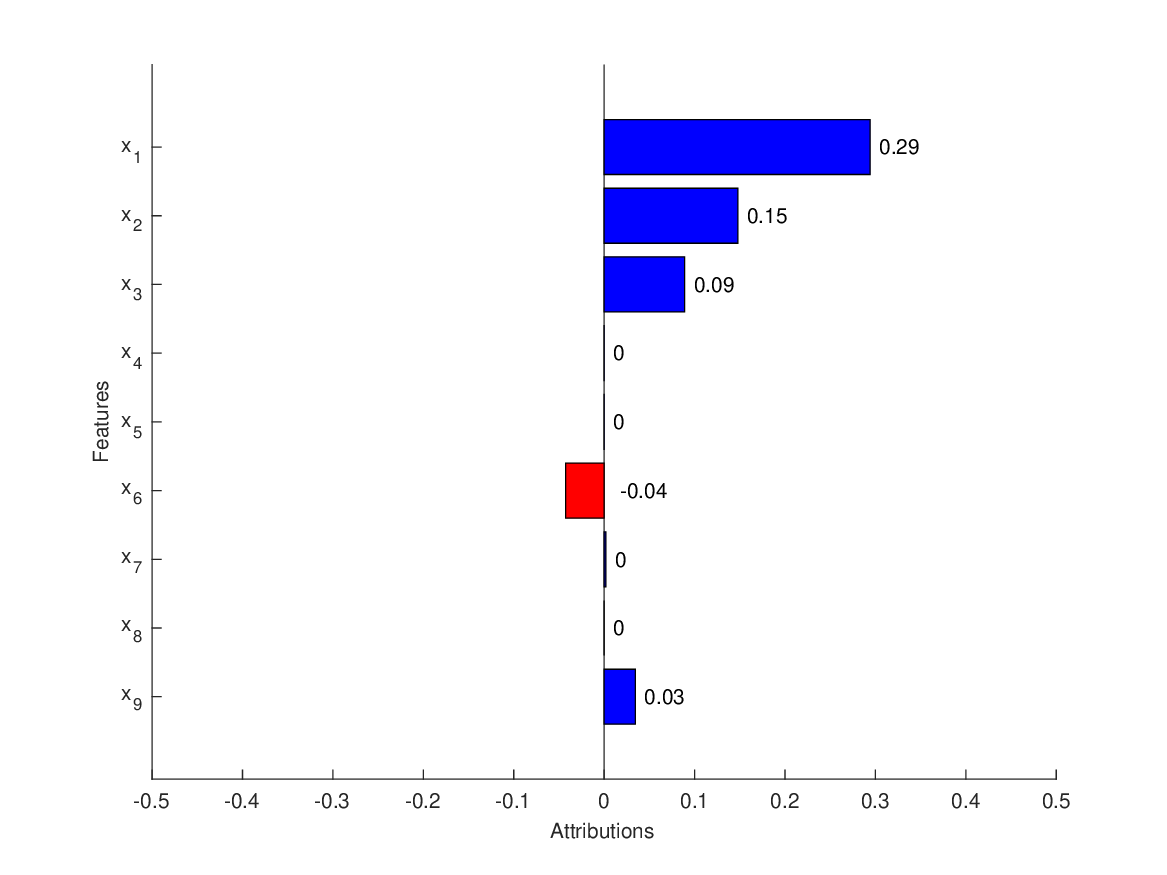}
    \caption{Instance explanation for the perturbed explicand (adding one $x_3$) by BShap for FCNN.  }
    \label{fig:Bshap_eg3}
\end{figure}

\subsection{MGNAM Explanations}

We now check our results using MGNAM, where monotonicity is enforced. Additionally, MGNAM is transparent, which allows us to examine explanations in greater detail. 
Due to the fact that $x_4-x_9$ are modeled by one-dimensional functions, attributions only represent values at the explicand, so they are not very interesting. Because of this, we focus on the result for $x_1-x_3$ and our focus is on $[0,2]$ for simplicity's sake. Results by IG and BShap for all cases are plotted in Figure~\ref{fig:IG_eg4} and Figure~\ref{fig:Bshap_eg4}. Overall, they perform similarly. All individual monotonicity is satisfied on average. However, Bshap can be problematic when attempting to achieve pairwise monotonicity. In the case of $(2,1,0)$, $\frac{1}{2} \text{BS}_1 < \text{BS}_2$, ASPM is violated by BShap. Meanwhile, IG preserves ASPM. Furthermore, we would expect to see potential more violations by BShap for larger values of $x_1-x_3$. Here, MGNAM has already preserved monotonicity. IG provides appropriate explanations, whereas BShap provides misleading explanations. Due to this, IG appears to be superior in explaining pairwise monotonicity. 




\begin{figure}
    \centering
    \includegraphics[scale=0.4]{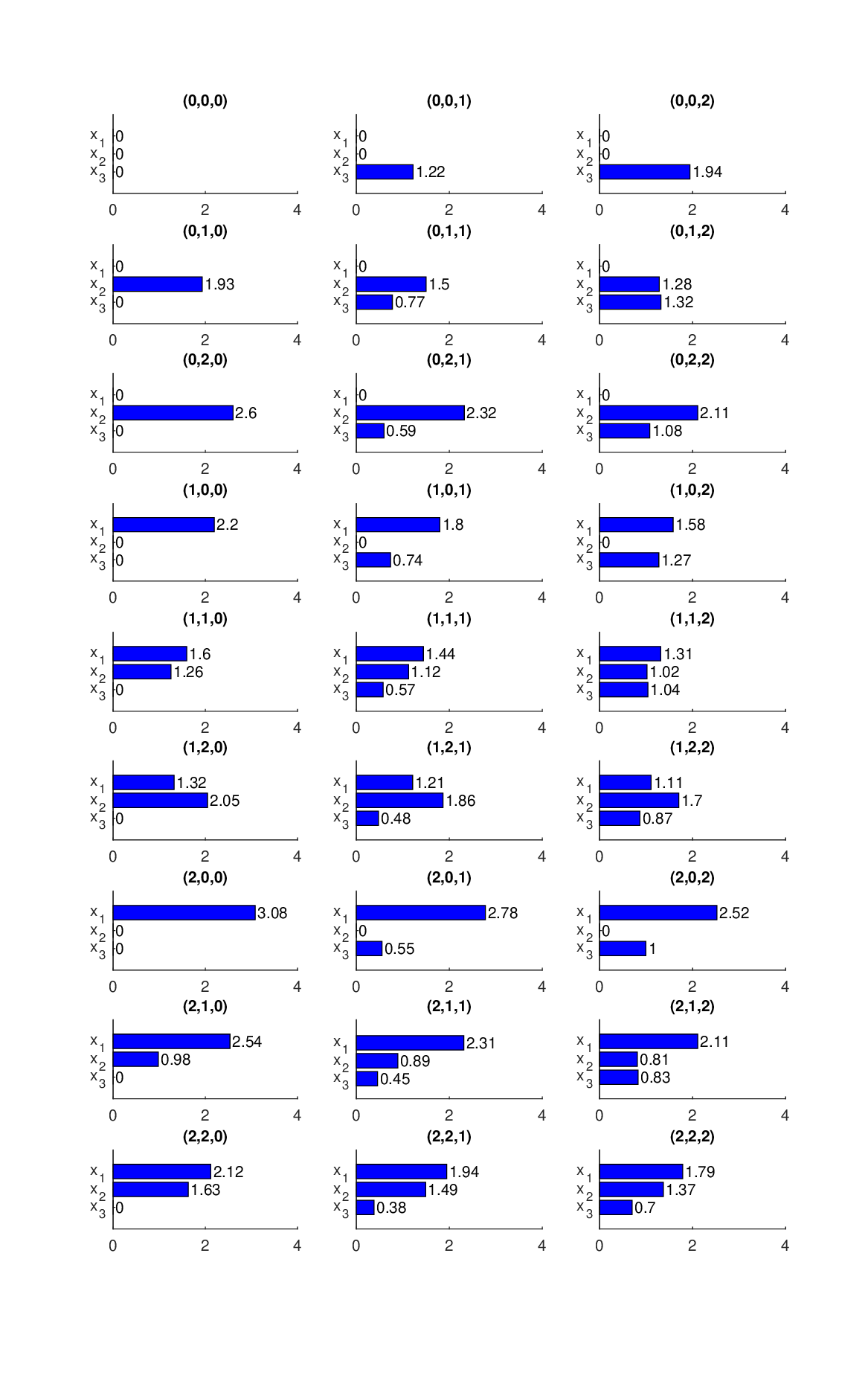}
    \caption{IG explanation for MGNAM for features $x_1-x_3$ in $[0,2]$. X-axis are attributions and y-axis are features. }
    \label{fig:IG_eg4}
\end{figure}

\begin{figure}
    \centering
    \includegraphics[scale=0.4]{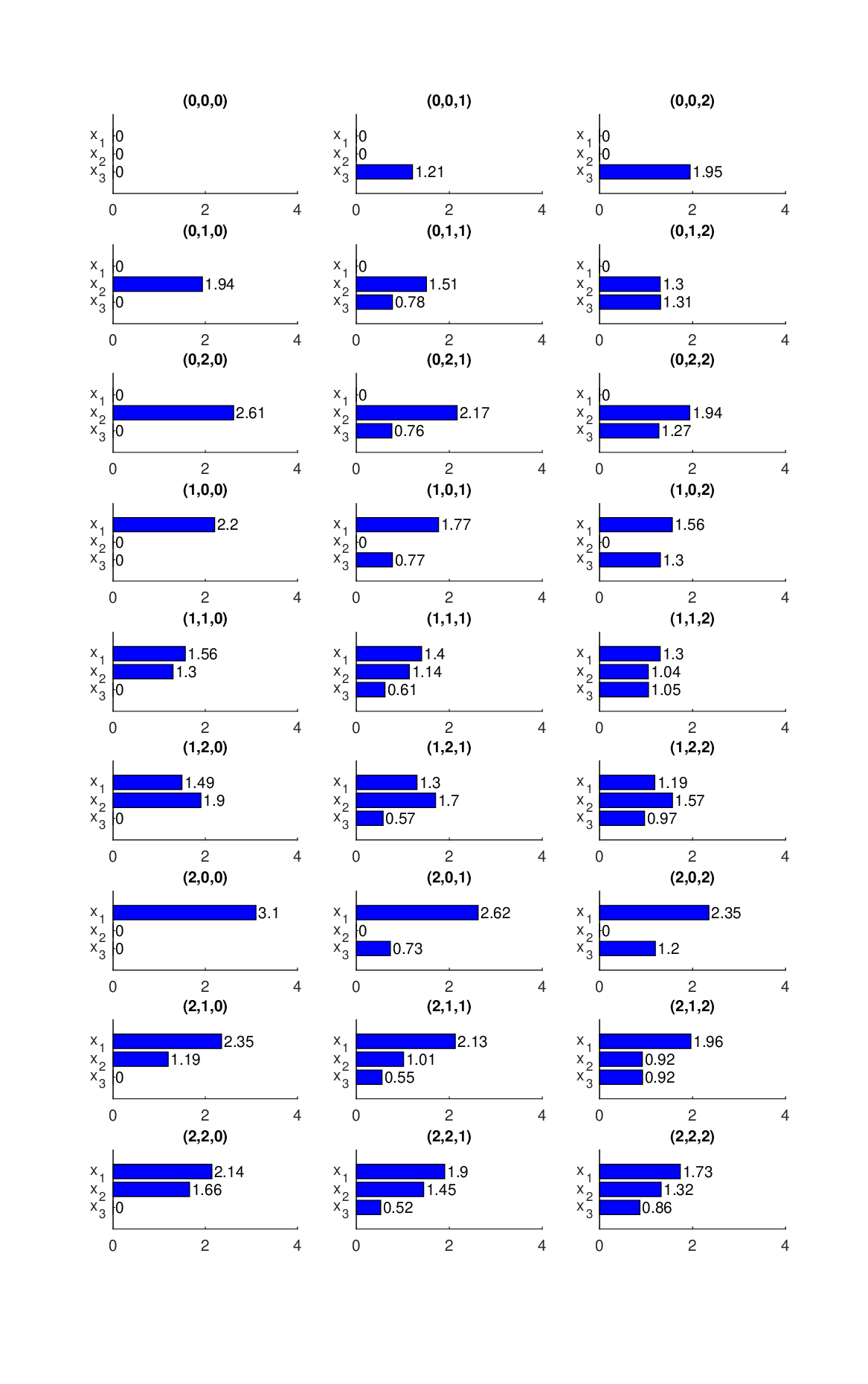}
    \caption{BShap explanation for MGNAM for features $x_1-x_3$ in $[0,2]$. X-axis are attributions and y-axis are features. At $(2,1,0)$, average strong pairwise monotonicity is violated. }
    \label{fig:Bshap_eg4}
\end{figure}

\section{Conclusion}

We investigate attribution methods for monotonic models in this paper. New axioms are proposed which must be satisfied by attribution methods. We analyze IG and BShap methods. By analyzing and illustrating, we demonstrate that BShap provides satisfactory explanations when only individual monotonicity is involved, but IG provides reasonable explanations when strong pairwise monotonicity is involved. 

In the future, we would like to pursue several directions. (1) We have focused on monotonicity in this paper, but it is necessary to investigate more domain knowledge in order to gain a better understanding of the behavior of ML models. (2) According to our analysis, neither IG nor BShap satisfies all monotonic axioms. It is necessary to modify existing methods or propose new methods in order to meet monotonicity requirements. (3) IG satisfies all monotonicity axioms on average, however, it is necessary that the underlying models are differentiable. Additionally, features such as the number of defaults are naturally discrete. Although IG can still be applied in this case using neural networks, it uses function values at points that are not useful in practice. For this reason, it is desirable to have a modified Shapley value method that is able to preserve strong pairwise monotonicity. 

In general, we believe that domain knowledge, such as monotonicity, plays an essential role in the interpretation of models. In sectors with high stakes, such as the finance sector, we must take these factors into account when we explain the models. 

\bibliographystyle{ACM-Reference-Format}
\bibliography{sample-base}

\end{document}